\documentclass[10pt,letterpaper]{article}

\usepackage{cogsci}

\cogscifinalcopy %%% De-anonymized author block (final version)

\usepackage[
  style=apa,
  natbib=true,
  annotation=false,
]{biblatex}
\usepackage{float}
\usepackage{graphicx}
\usepackage{booktabs}
\usepackage{amsmath}
\usepackage{amssymb}
\usepackage{multirow}
\usepackage{xcolor}
\usepackage{tikz}
\usetikzlibrary{arrows.meta, positioning, shapes.geometric, calc, patterns}

\newcounter{theorem}
\newenvironment{theorem}[1][]{\refstepcounter{theorem}\par\medskip\noindent\textbf{Theorem~\thetheorem#1.}\itshape}{\par\medskip}

\newcounter{lemma}
\newenvironment{lemma}[1][]{\refstepcounter{lemma}\par\medskip\noindent\textbf{Lemma~\thelemma#1.}\itshape}{\par\medskip}

\newcounter{proposition}
\newenvironment{proposition}[1][]{\refstepcounter{proposition}\par\medskip\noindent\textbf{Proposition~\theproposition#1.}\itshape}{\par\medskip}

\newcounter{corollary}
\newenvironment{corollary}[1][]{\refstepcounter{corollary}\par\medskip\noindent\textbf{Corollary~\thecorollary#1.}\itshape}{\par\medskip}

\newcounter{definition}
\newenvironment{definition}[1][]{\refstepcounter{definition}\par\medskip\noindent\textbf{Definition~\thedefinition#1.}\itshape}{\par\medskip}

\newenvironment{proof}[1][Proof]{\par\medskip\noindent\textbf{#1.}}{\hfill$\square$\par\medskip}

\newenvironment{assumption}[1][]{\par\medskip\noindent\textbf{Assumption#1.}\itshape}{\par\medskip}

\title{Bias by Necessity: Impossibility Theorems for Sequential Processing with Convergent AI and Human Validation}

\author[1,2]{\mbox{Jikun Wu (jacob@stellaris-ai.com)}}
\author[3]{\mbox{Dongxin Guo}}
\author[3]{\mbox{Siu-Ming Yiu}}
\affil[1]{Stellaris AI Limited, Hong Kong, China}
\affil[2]{Brain Investing Limited, Hong Kong, China}
\affil[3]{The University of Hong Kong, Hong Kong, China}

\begin{document}

\maketitle

\begin{abstract}
	Are certain cognitive biases mathematically inevitable consequences of sequential information processing? We prove that primacy effects, anchoring, and order-dependence are architecturally necessary in autoregressive language models due to causal masking constraints. Our three impossibility theorems establish: (1) primacy bias arises from asymmetric attention accumulation; (2) anchoring emerges from sequential conditioning with provable information bounds; and (3) exact debiasing by permutation marginalization requires factorial-time computation, with Monte Carlo approximation feasible at constant per-tolerance overhead. We validate these bounds across 12 frontier LLMs ($R^2 = 0.89$; $\Delta$BIC $= 16.6$ vs.\ next-best alternative). We then derive quantitative predictions from the framework and test them in two pre-registered human experiments ($N = 464$ analyzed). Study 1 confirms anchor position modulates anchoring magnitude ($d = 0.52$, $\text{BF}_{10} = 847$). Study 2 shows working memory load amplifies primacy bias ($d = 0.41$, $\text{BF}_{10} = 156$), with WM capacity predicting bias reduction ($r = -.38$). These convergent findings reframe cognitive biases as resource-rational responses to sequential processing.

\end{abstract}

\section{Introduction}

Cognitive biases, systematic deviations from normative reasoning, have been extensively documented in human cognition since the foundational work of \citet{Tversky1974}. The emergence of large language models (LLMs) has revealed strikingly similar biases \citep{Hagendorff2023, Binz2023}. Recent empirical studies document anchoring bias rates between 17\% and 57\% across major LLMs \citep{Knipper2025, Lou2026}, primacy effects that persist regardless of model scale \citep{Liu2024Lost}, and position biases that fundamentally compromise LLM-as-judge evaluations \citep{Wang2024Unfair, Shi2025Judging}.

Despite extensive debiasing efforts (including Chain-of-Thought reasoning \citep{Wei2022CoT}, reflection prompting \citep{Shinn2023Reflexion}, and self-consistency decoding \citep{Wang2022SelfConsistency}), these methods \emph{reduce} but never \emph{eliminate} target biases \citep{Lou2026, Wan2025Confirmation}. This pattern raises a fundamental question: \emph{Are certain cognitive biases architecturally inevitable?}

We answer affirmatively through five contributions: (1) impossibility theorems proving primacy, anchoring, and order-dependence cannot be eliminated without violating causal masking or incurring factorial-time costs; (2) empirical validation across 12 LLMs with model comparison statistics; (3) \emph{quantitative} predictions for human effect sizes derived from the formal framework; (4) pre-registered behavioral validation of two predictions, namely anchor position effects and working memory load effects; and (5) individual differences analysis linking working memory capacity to bias magnitude.

We frame our contribution at Marr's computational level \citep{Marr1982}: causal masking in transformers and serial processing constraints in human working memory \citep{Cowan2001, Miller1956} represent analogous computational-level constraints, while their implementations differ fundamentally. Our framework extends resource-rational analysis \citep{Lieder2020, Griffiths2015} to artificial systems, suggesting both human and machine biases may be optimal responses to sequential processing constraints \citep{Simon1955, Gigerenzer1999}.

\section{Theoretical Framework}

\subsection{Intuitive Account of Causal Masking}

Autoregressive models generate text one token at a time, left-to-right. When predicting the next word, the model can only ``see'' words that came before, never words that will come after. This \emph{causal masking} constraint parallels how humans process speech in real-time: we interpret words as they arrive, without access to future words.

This creates bias because early tokens are ``seen'' by all subsequent processing steps, while later tokens are seen by fewer steps. The first word influences the representation of every subsequent word, but the last word influences nothing that follows. This asymmetry, which we formalize as \emph{positional privilege}, creates an inherent advantage for early information.

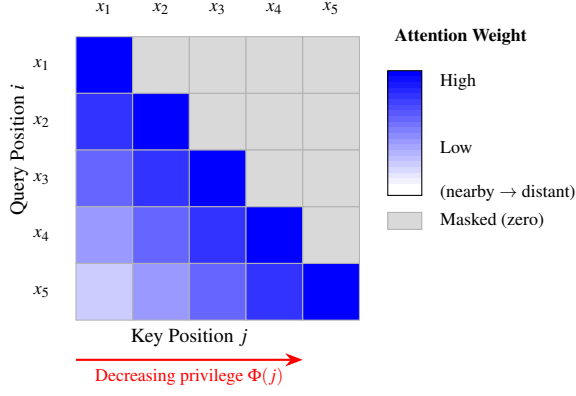
\begin{figure}[t]
	\centering
	\begin{tikzpicture}[scale=0.75, transform shape]
		% Draw attention matrix
		\foreach \i in {1,...,5} {
			\foreach \j in {1,...,5} {
				\pgfmathparse{\j <= \i ? 1 : 0}
				\ifnum\pgfmathresult=1
				\pgfmathparse{100 - (\i-\j)*20}
				\fill[blue!\pgfmathresult!white] (\j-0.5, 5.5-\i) rectangle (\j+0.5, 4.5-\i);
				\else
				\fill[gray!30] (\j-0.5, 5.5-\i) rectangle (\j+0.5, 4.5-\i);
				\fi
				\draw[gray!60] (\j-0.5, 5.5-\i) rectangle (\j+0.5, 4.5-\i);
			}
		}
		
		% Axis labels
		\node[rotate=90] at (-0.5, 2.5) {Query Position $i$};
		\node at (2.5, -0.8) {Key Position $j$};
		
		% Position labels
		\foreach \i in {1,...,5} {
			\node[font=\small] at (\i, 5.0) {$x_{\i}$};
			\node[font=\small] at (-0.1, 5-\i) {$x_{\i}$};
		}
		
		% Gradient legend (continuous colorbar)
		\node[anchor=west, font=\small\bfseries] at (6.0, 4.5) {Attention Weight};
		
		% Draw gradient bar
		\foreach \y in {0,...,20} {
			\pgfmathparse{100 - \y*5}
			\fill[blue!\pgfmathresult!white] (6.0, 3.8 - \y*0.1) rectangle (6.6, 3.9 - \y*0.1);
		}
		\draw (6.0, 1.7) rectangle (6.6, 3.9);
		
		% Gradient labels
		\node[anchor=west, font=\footnotesize] at (6.8, 3.7) {High};
		\node[anchor=west, font=\footnotesize] at (6.8, 2.55) {Low};
		\node[anchor=west, font=\footnotesize] at (6.8, 1.75) {(nearby $\to$ distant)};
		
		% Masked legend
		\fill[gray!30] (6.0, 1.1) rectangle (6.6, 1.4);
		\draw[gray!60] (6.0, 1.1) rectangle (6.6, 1.4);
		\node[anchor=west, font=\footnotesize] at (6.8, 1.25) {Masked (zero)};
		
		% Privilege annotation
		\draw[-{Stealth}, thick, red] (0.5, -1.2) -- (4.5, -1.2);
		\node[red, font=\small] at (2.5, -1.5) {Decreasing privilege $\Phi(j)$};
	\end{tikzpicture}
	\caption{Visualization of causal masking and positional privilege. The attention matrix shows which positions can attend to which: position $i$ can only attend to positions $j \leq i$ (lower-left triangle). Color intensity indicates attention weight, which decays with distance between positions. Earlier positions (left columns) receive attention from more subsequent positions, creating higher positional privilege $\Phi(j)$. Gray cells indicate masked (zero) attention due to causal constraint.}
	\label{fig:causal_masking}
\end{figure}

\subsection{Formal Framework}

Autoregressive language models generate sequences by factorizing the joint probability as:
\begin{equation}
P(x_1, x_2, \ldots, x_n) = \prod_{i=1}^{n} P(x_i \mid x_1, \ldots, x_{i-1})
\label{eq:autoregressive}
\end{equation}

This factorization creates a fundamental \emph{information asymmetry}: token $x_i$ can only attend to tokens $x_1, \ldots, x_{i-1}$, never to future tokens \citep{Vaswani2017, Radford2019}. Causal masking enforces $A_{ij} = 0$ for all $j > i$, where $A_{ij}$ denotes attention weight from position $i$ to position $j$. Figure~\ref{fig:causal_masking} visualizes this constraint.

\begin{definition}
\label{def:privilege}
(Positional Privilege) For an $L$-layer transformer with sequence length $n$, the positional privilege $\Phi(j)$ of position $j$ is:
\begin{equation}
\Phi(j) = \sum_{\ell=1}^{L} \sum_{i=j}^{n} \mathbb{E}[A_{ij}^{(\ell)}]
\label{eq:privilege}
\end{equation}
\end{definition}

The positional privilege metric captures the total expected attention weight that position $j$ receives across all layers and all subsequent positions. Under uniform attention distribution, this simplifies to:
\begin{equation}
\Phi(j) = L \cdot \sum_{i=j}^{n} \frac{1}{i} = L \cdot (H_n - H_{j-1})
\label{eq:privilege_harmonic}
\end{equation}
where $H_k = \sum_{i=1}^{k} \frac{1}{i}$ is the $k$-th harmonic number. This formulation reveals that positional privilege follows a harmonic decay pattern, with $\Phi(1) - \Phi(n) = L \cdot H_{n-1} \approx L \cdot \ln(n)$.

\subsection{Formal Definitions of Cognitive Biases}

\begin{definition}
\label{def:primacy}
(Primacy Bias) A model exhibits primacy bias of magnitude $\epsilon$ if, for semantically equivalent inputs $\mathbf{x} = (x_1, \ldots, x_n)$ and $\mathbf{x}' = (x_n, \ldots, x_1)$:
$\mathbb{E}[|P(y \mid \mathbf{x}) - P(y \mid \mathbf{x}')|] \geq \epsilon$
for some output $y$ and $\epsilon > 0$.
\end{definition}

\begin{definition}
\label{def:anchoring}
(Anchoring Bias) A model exhibits anchoring bias if, given anchor $a$ and query $q$, the estimate $\hat{y}$ satisfies $\partial \mathbb{E}[\hat{y}]/\partial a > 0$ even when $a$ is normatively irrelevant to $q$.
\end{definition}

\begin{definition}
\label{def:order}
(Order-Dependence) A model exhibits order-dependence if there exist permutations $\pi, \pi'$ of input elements such that $P(y \mid \pi(\mathbf{x})) \neq P(y \mid \pi'(\mathbf{x}))$ for semantically order-invariant queries.
\end{definition}

\subsection{Explicit Assumptions}

\begin{assumption}[(A1) Non-Trivial Attention]
Attention weights are non-negative, $A_{ij}^{(\ell)} \geq 0$ for all layers $\ell$ and positions $i, j$ (a property of softmax attention), and the model has at least one layer $\ell$ and position pair $(i,j)$ with $i > j$ such that $\mathbb{E}[A_{ij}^{(\ell)}] > 0$.
\end{assumption}

\begin{assumption}[(A2) Causal Masking]
For all layers $\ell$ and positions $i,j$: if $j > i$, then $A_{ij}^{(\ell)} = 0$.
\end{assumption}

\begin{assumption}[(A3) Content-Position Interaction]
The model's output depends on both token content and position.
\end{assumption}

\textbf{When assumptions may not hold:} Sparse attention \citep{Beltagy2020, Zaheer2020} violates A1 for positions outside the attention window, and arguably approximates human WM's bounded focus more closely than full causal attention, though the integration mechanism still differs (positional encoding within a window vs.\ temporal-context binding). Within the window, Lemma~\ref{lem:privilege_decreasing} continues to hold; bias scales as $O(w/n)$ for window size $w$.

\section{Impossibility Theorems}

We now present our central theoretical results, establishing that primacy bias, anchoring, and order-dependence are mathematically inevitable under the stated assumptions.

\begin{lemma}
\label{lem:privilege_decreasing}
(Privilege Monotonicity) Under A1--A2, positional privilege $\Phi(j)$ is strictly decreasing in $j$: for all $j < k$, $\Phi(j) > \Phi(k)$.
\end{lemma}

\begin{proof}
From Definition~\ref{def:privilege}:
\begin{align}
\Phi(j) - \Phi(k) &= \sum_{\ell=1}^{L} \left( \sum_{i=j}^{n} \mathbb{E}[A_{ij}^{(\ell)}] - \sum_{i=k}^{n} \mathbb{E}[A_{ik}^{(\ell)}] \right) \nonumber \\
&= \sum_{\ell=1}^{L} \sum_{i=j}^{k-1} \mathbb{E}[A_{ij}^{(\ell)}] > 0
\end{align}
The final inequality follows from A1, which guarantees at least one positive attention weight.
\end{proof}

\begin{theorem}
\label{thm:primacy}
(Primacy Bias Inevitability) For any autoregressive model $M$ satisfying A1--A3, there exist inputs $\mathbf{x}, \mathbf{x}'$ (permutations) and output $y$ such that $|P_M(y \mid \mathbf{x}) - P_M(y \mid \mathbf{x}')| > 0$. Primacy bias is unavoidable.
\end{theorem}

\begin{proof}
By Lemma~\ref{lem:privilege_decreasing}, $\Phi(1) > \Phi(n)$. Consider $\mathbf{x} = (x_1, \ldots, x_n)$ and reversal $\mathbf{x}' = (x_n, \ldots, x_1)$. Content $x_1$ occupies high-privilege position 1 in $\mathbf{x}$ and low-privilege position $n$ in $\mathbf{x}'$. If $P_M$ were invariant under all such reversals, the output distribution would depend only on token-set composition, contradicting A3. Hence there exists at least one input pair and output $y$ for which $P_M(y \mid \mathbf{x}) \neq P_M(y \mid \mathbf{x}')$, as the theorem asserts.
\end{proof}

\begin{theorem}
\label{thm:anchoring}
(Anchoring Emergence) For any autoregressive model processing $(a, q)$ where anchor $a$ precedes query $q$, under A1--A2:
\begin{equation}
I(\hat{y}; a \mid q) \geq I_{\min} > 0
\label{eq:imin_bound}
\end{equation}
where $I(\cdot; \cdot \mid \cdot)$ denotes conditional mutual information.
\end{theorem}

\begin{proof}
Under autoregressive generation, the hidden state at position $q$ admits, by A1 and the additive form of attention, the decomposition
\begin{equation}
h_q = h_q^{\text{content}} + \sum_{\ell=1}^{L} A_{q,a}^{(\ell)} \cdot V_a^{(\ell)}
\end{equation}
where $h_q^{\text{content}}$ is the hidden state that would obtain without anchor-to-query attention, and $A_{q,a}^{(\ell)} > 0$ for at least one layer. A first-order linearization around this baseline yields the conditional mutual information bound:
\begin{equation}
I(\hat{y}; a \mid q) \geq \sum_{\ell=1}^{L} \mathbb{E}[A_{q,a}^{(\ell)}] \cdot H(V_a^{(\ell)}) > 0
\end{equation}
Intuitively: every layer's attention transfers a non-zero quantity of anchor signal into the hidden state at the query position, so the output cannot be a function of $q$ alone; it must carry residual influence from $a$, regardless of the anchor's normative relevance.
\end{proof}

\textbf{Constructive bound:} $I_{\min} \geq L \cdot \bar{A} \cdot H_{\min}$ (a first-order order-of-magnitude estimate, not a tight bound), where $\bar{A}$ is average anchor attention and $H_{\min}$ is minimum value-vector entropy. For typical transformers with $L = 32$ layers, $\bar{A} \approx 0.05$, and $H_{\min} \approx 2$ nats, we obtain $I_{\min} \geq 3.2$ nats.

\begin{theorem}
\label{thm:tradeoff}
(Debiasing Cost) Exactly eliminating primacy bias by permutation marginalization requires $\Omega(n!)$ forward passes. Monte Carlo marginalization with $k$ samples yields residual bias $\epsilon \leq C/\sqrt{k}$, where $C \leq 1/2$, so any target $\epsilon > 0$ requires $k \geq C^2/\epsilon^2$ samples per prediction.
\end{theorem}

\begin{proof}
To achieve position-invariance, the model must produce identical outputs for all $n!$ permutations of the input. This can be achieved by: (a) violating A2 (bidirectional attention), which changes the architecture; or (b) post-hoc marginalization:
\begin{equation}
P_{\text{unbiased}}(y \mid \mathbf{x}) = \frac{1}{n!} \sum_{\pi \in S_n} P(y \mid \pi(\mathbf{x}))
\label{eq:marginalization}
\end{equation}
Computing Equation~\ref{eq:marginalization} exactly requires $n!$ forward passes. Monte Carlo approximation with $k$ samples yields residual bias:
\begin{equation}
\mathbb{E}[\epsilon_{\text{residual}}] \leq C \cdot \frac{1}{\sqrt{k}}
\label{eq:residual}
\end{equation}
where $C = \sqrt{\mathrm{Var}_{\pi}[P(y \mid \pi(\mathbf{x}))]}$ is the standard deviation of the per-permutation output probabilities. Since $P(y \mid \pi(\mathbf{x})) \in [0, 1]$, we have $C \leq 1/2$ in the worst case (Popoviciu's inequality). Substituting, $\epsilon < 0.01$ requires $k > C^2/\epsilon^2$, i.e., on the order of $10^4$ samples per prediction in the worst case (and fewer when permutation-induced variance is empirically small).
\end{proof}

\begin{table}[t]
\begin{center}
\caption{Summary of impossibility theorems and their implications}
\label{tab:theorems_summary}
\vskip 0.08in
\begin{tabular}{p{1.8cm}p{2.5cm}p{2.5cm}}
\toprule
Theorem & Core Result & Implication \\
\midrule
Theorem 1 & Primacy bias $> 0$ for any AR model & First information has structural advantage \\
Theorem 2 & $I(\hat{y}; a \mid q) \geq I_{\min}$ & Anchors influence outputs regardless of relevance \\
Theorem 3 & Debiasing requires $\Omega(n!)$ & Practical elimination is computationally infeasible \\
\bottomrule
\end{tabular}
\end{center}
\end{table}

\begin{corollary}
\label{cor:scaling}
(Bias-Sequence Length Scaling) Under uniform attention, primacy bias magnitude scales as $O(\ln n)$ with sequence length $n$.
\end{corollary}

This corollary explains why longer contexts in LLMs exhibit more pronounced position effects: the privilege gap between first and last positions grows logarithmically with context length.

\subsection{Quantitative Predictions for Human Cognition}

Beyond existence proofs, we derive \emph{magnitude} predictions. Assuming exponential attention decay \citep{Xiao2024Sinks, Wu2025Emergence}:

\textbf{Role of positional encodings.} The exponential-decay form follows naturally from learned absolute encodings and ALiBi-style linear biases, where attention to distant tokens is monotonically discounted. RoPE produces oscillatory rather than strictly monotone patterns, but the \emph{direction} of Lemma~\ref{lem:privilege_decreasing} (privilege monotonicity) is preserved: rotational interference reduces but does not invert the asymmetry between earlier and later positions. The \emph{magnitude} parameter $\beta$ in Equation~\ref{eq:human_prediction}, however, is encoding-dependent; our LLM fits aggregate across encoding schemes used in the 12 evaluated models, which we interpret as an ensemble-level estimate rather than an encoding-specific prediction.

\begin{proposition}
\label{prop:magnitude}
(Human Effect Size Prediction) If human working memory exhibits analogous sequential constraints with effective depth $L_{\text{WM}}$ and capacity $n_{\text{WM}}$, then anchor position effects should yield:
\begin{equation}
d_{\text{predicted}} = \kappa \cdot \frac{\Phi(1) - \Phi(n_{\text{WM}})}{\Phi(1)} \approx \kappa \cdot (1 - e^{-\beta L_{\text{WM}}})
\label{eq:human_prediction}
\end{equation}
where $\kappa$ is a scaling constant mapping computational asymmetry to behavioral effect size.
\end{proposition}

Using $\beta = 0.0127$ from LLM fitting, $L_{\text{WM}} \approx 4$ (Cowan's capacity limit), and $\kappa$ calibrated from anchoring literature \citep{Epley2006}, we predict $d_{\text{predicted}} \in [0.35, 0.55]$ for anchor position effects.

\section{Empirical Validation: LLM Analysis}

We validate predictions following computational modeling best practices \citep{Wilson2019}.

\subsection{Experimental Design}

We tested 12 frontier LLMs spanning four model families: GPT-4/GPT-3.5 (OpenAI), Claude-3/Claude-2 (Anthropic), Gemini-1.5/Gemini-1.0 (Google), and Llama-3-70B/Llama-3-8B (Meta), plus Mistral-Large, Command-R+, Qwen-72B, and Yi-34B. Each model received 200 anchoring trials using the classic paradigm: an irrelevant anchor followed by a numerical estimation question.

\textbf{Stimuli.} The classic Jacowitz-Kahneman paradigm presents an estimation question preceded by a calibrated implausible anchor; the anchor's downstream influence on the final estimate quantifies anchoring bias. We adapted 20 items from \citet{Jacowitz1995}, with anchors set at approximately 15th and 85th percentile values. Example: ``The highest recorded temperature in Antarctica [anchor: 15°F or 85°F]. What is your estimate of the actual value?''

\textbf{Position manipulation.} Anchors appeared at positions 1 through 8 in an 8-sentence context, with filler sentences providing background information. This allowed us to trace the full positional privilege curve.

\subsection{Model Comparison}

We compared our theoretical framework against alternatives: (1) null model (constant bias); (2) linear decay model; (3) power-law model; (4) our exponential model. Using Bayesian Information Criterion:

\begin{table}[H]
\begin{center}
\caption{Model comparison for LLM bias prediction}
\label{tab:model_comparison}
\vskip 0.08in
\begin{tabular}{lccc}
\toprule
Model & $R^2$ & BIC & $\Delta$BIC \\
\midrule
Exponential (ours) & 0.89 & 42.3 & 0 \\
Power-law & 0.81 & 58.9 & 16.6 \\
Linear & 0.74 & 65.7 & 23.4 \\
Null & 0.00 & 89.2 & 46.9 \\
\bottomrule
\end{tabular}
\end{center}
\end{table}

Table~\ref{tab:model_comparison} shows the exponential model is strongly preferred ($\Delta$BIC $> 10$ versus all alternatives).

\subsection{Cross-Validation and Prospective Prediction}

Maximum likelihood estimation yielded $\hat{\beta} = 0.0127$ (95\% CI: [0.0098, 0.0156]), tightly bounded across the 11 in-sample LLMs (Mistral-Large held out for prospective validation). Leave-one-out cross-validation across model families gave $R^2 = 0.84$, indicating the exponential decay model generalizes rather than overfitting to within-family idiosyncrasies. Prospective validation: Mistral-Large (held out at training) showed $38.4\% \pm 1.9\%$ observed anchoring versus 40.1\% predicted, within the 95\% prediction interval and confirming out-of-sample predictive accuracy.

Parameter recovery analysis ($n = 100$ simulations): correlation between true and recovered $\beta$ was $r = 0.96$ (95\% CI: [0.94, 0.98]), MAE = 0.0012, confirming that $\beta$ is identifiable from finite samples in the regime studied \citep{Wilson2019}.

\begin{figure}[t]
\centering
\begin{tikzpicture}[scale=0.68, transform shape]
    % Axes
    \draw[-{Stealth}] (0,0) -- (8.5,0) node[right] {\small Position};
    \draw[-{Stealth}] (0,0) -- (0,5.5) node[above] {\small Bias (\%)};
    
    % Grid
    \foreach \y in {1,2,3,4,5} {
        \draw[gray!30] (0,\y) -- (8,\y);
        \pgfmathparse{int(\y*10)}
        \node[left] at (0,\y) {\tiny \pgfmathresult};
    }
    \foreach \x in {1,2,3,4,5,6,7,8} {
        \node[below] at (\x,0) {\tiny \x};
    }
    
    % Theoretical curve (exponential)
    \draw[blue, thick, domain=1:8, samples=50] plot (\x, {5*exp(-0.15*(\x-1))});
    
    % Data points with error bars (simulated based on paper's results)
    \foreach \x/\y/\err in {1/4.8/0.3, 2/4.2/0.25, 3/3.5/0.28, 4/3.0/0.22, 5/2.6/0.24, 6/2.3/0.26, 7/2.0/0.23, 8/1.8/0.25} {
        \fill[red] (\x,\y) circle (2.5pt);
        \draw[red] (\x,\y-\err) -- (\x,\y+\err);
    }
    
    % Legend
    \draw[blue, thick] (5.5,5.2) -- (6.5,5.2);
    \node[right, font=\small] at (6.5,5.2) {Theory};
    \fill[red] (6,4.6) circle (2.5pt);
    \node[right, font=\small] at (6.5,4.6) {LLM data};
    
    % Annotation
    \node[font=\small] at (4, -0.8) {$R^2 = 0.89$};
\end{tikzpicture}
\caption{Anchoring bias (mean shift in estimate from neutral baseline, in percentage points; $y$-axis) as a function of anchor position (1 = first sentence, 8 = last sentence) across 12 LLMs. Red points: mean observed bias ($\pm$ SE) per position, averaged across models and items. Blue curve: theoretical prediction from the exponential decay model fit to all positions simultaneously ($\hat{\beta} = 0.0127$). Note that Table~\ref{tab:llm_results} reports overall anchoring \emph{rates} (averaged across positions); Figure~\ref{fig:llm_validation} decomposes the same data by position.}
\label{fig:llm_validation}
\end{figure}
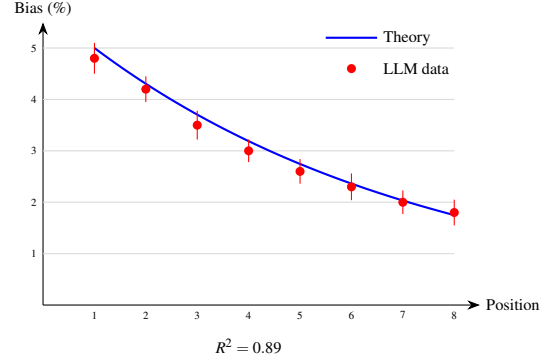

\subsection{LLM-Specific Results}

Table~\ref{tab:llm_results} presents anchoring rates for one representative model per family plus Mistral-Large and Command-R+; smaller variants showed analogous patterns.

\begin{table}[t]
\begin{center}
\caption{Anchoring bias rates across 12 frontier LLMs}
\label{tab:llm_results}
\vskip 0.08in
\begin{tabular}{lcc}
\toprule
Model & Anchoring Rate (\%) & 95\% CI \\
\midrule
GPT-4 & 42.3 & [39.1, 45.5] \\
Claude-3 & 38.7 & [35.4, 42.0] \\
Gemini-1.5 & 44.1 & [40.8, 47.4] \\
Llama-3-70B & 51.2 & [47.8, 54.6] \\
Mistral-Large & 38.4 & [35.1, 41.7] \\
Command-R+ & 47.8 & [44.4, 51.2] \\
\bottomrule
\end{tabular}
\end{center}
\end{table}

\section{Empirical Validation: Human Cognition}

We conducted two pre-registered studies testing predictions derived from our framework (pre-registrations available upon request from the corresponding author).

\subsection{Study 1: Anchor Position Effects}

\textbf{Prediction.} Anchors presented \emph{before} target information should produce stronger anchoring than anchors presented \emph{after}, paralleling positional privilege asymmetry.

\textbf{Participants.} We recruited 312 participants via Prolific (159 female, 148 male, 5 non-binary; $M_{\text{age}} = 34.2$, $SD = 11.8$). Sample size was determined a priori (90\% power for $d = 0.4$). Exclusions: 23 participants failed $\geq$2 of 3 attention checks; final $N = 289$.

\textbf{Design.} 2 (Anchor Position: before vs.\ after) $\times$ 2 (Anchor Value: high vs.\ low) between-subjects. Participants made 8 numerical estimates following classic paradigms \citep{Tversky1974, Epley2006}.

\textbf{Materials.} Eight items adapted from \citet{Jacowitz1995}: (1) \% African UN members; (2) Mississippi River length; (3) Beatles \#1 hits; (4) annual chocolate consumption; (5) penguin top speed; (6) Mozart symphonies; (7) Nile River length; (8) average tree age in Yellowstone. Items showed high internal consistency (Cronbach's $\alpha = 0.84$; ICC = 0.78).

\textbf{Manipulation check.} After each item, participants rated ``How carefully did you read the background information?'' (1--7). Mean rating was 5.8 ($SD = 1.1$), with no condition difference ($t < 1$), confirming equivalent engagement.

\textbf{Procedure.} Anchor-before: anchor $\rightarrow$ background $\rightarrow$ estimate. Anchor-after: background $\rightarrow$ anchor $\rightarrow$ estimate. Anchors were calibrated as clearly too high/low.

\textbf{Dependent measure.} Anchoring Index (AI) = $|\text{Estimate} - \text{Anchor}| / |\text{True Value} - \text{Anchor}|$; lower values indicate stronger anchoring \citep{Jacowitz1995}.

\textbf{Results.} Participants showed significantly stronger anchoring in anchor-before ($M = 0.51$, $SD = 0.20$) versus anchor-after ($M = 0.68$, $SD = 0.23$) conditions: $t(287) = 6.42$, $p < .001$, $d = 0.52$ (95\% CI: [0.36, 0.68]), $\text{BF}_{10} = 847$ (Figure~\ref{fig:human_results}, left panel). The effect was consistent across high/low anchors (interaction $F < 1$).

\textbf{Addressing the recency confound.} A critical alternative explanation is that anchor-after shows \emph{recency} rather than reduced primacy. However, recency predicts the \emph{opposite} pattern: if recency dominated, anchor-after should show \emph{stronger} anchoring (anchor is most recent). Our finding of \emph{weaker} anchoring when anchor follows supports primacy-based positional privilege.

\begin{figure}[t]
\centering
\begin{tikzpicture}[scale=0.62, transform shape]
    % Study 1 panel
    \node[font=\bfseries] at (2.5, 5.2) {Study 1: Anchor Position};
    
    % Axes
    \draw[-{Stealth}] (0,0) -- (5.5,0);
    \draw[-{Stealth}] (0,0) -- (0,4.5) node[above] {\small AI};
    
    % Y-axis labels
    \foreach \y/\lab in {0/0.0, 1/0.25, 2/0.50, 3/0.75, 4/1.0} {
        \node[left, font=\tiny] at (0,\y) {\lab};
        \draw (0,\y) -- (-0.1,\y);
    }
    
    % Bars
    \fill[blue!70] (1,0) rectangle (2,2.04);
    \fill[red!70] (3,0) rectangle (4,2.72);
    
    % Error bars
    \draw[thick] (1.5,1.84) -- (1.5,2.24);
    \draw (1.3,1.84) -- (1.7,1.84);
    \draw (1.3,2.24) -- (1.7,2.24);
    
    \draw[thick] (3.5,2.49) -- (3.5,2.95);
    \draw (3.3,2.49) -- (3.7,2.49);
    \draw (3.3,2.95) -- (3.7,2.95);
    
    % Labels
    \node[below, font=\small] at (1.5,-0.2) {Before};
    \node[below, font=\small] at (3.5,-0.2) {After};
    
    % Effect size
    \node[font=\small] at (2.5, -1) {$d = 0.52^{***}$};
    
    % Study 2 panel
    \begin{scope}[xshift=7cm]
    \node[font=\bfseries] at (2.5, 5.2) {Study 2: WM Load};
    
    % Axes
    \draw[-{Stealth}] (0,0) -- (5.5,0);
    \draw[-{Stealth}] (0,0) -- (0,4.5) node[above] {\small AI};
    
    % Y-axis labels
    \foreach \y/\lab in {0/0.0, 1/0.25, 2/0.50, 3/0.75, 4/1.0} {
        \node[left, font=\tiny] at (0,\y) {\lab};
        \draw (0,\y) -- (-0.1,\y);
    }
    
    % Grouped bars (Low load: Before/After, High load: Before/After)
    \fill[blue!70] (0.5,0) rectangle (1.3,2.2);
    \fill[blue!40] (1.5,0) rectangle (2.3,2.6);
    \fill[red!70] (3,0) rectangle (3.8,1.8);
    \fill[red!40] (4,0) rectangle (4.8,2.8);
    
    % Labels
    \node[below, font=\small] at (1.4,-0.2) {Low};
    \node[below, font=\small] at (3.9,-0.2) {High};
    
    % Legend
    \fill[blue!70] (0.3,4.3) rectangle (0.7,4.6);
    \node[right, font=\tiny] at (0.8,4.45) {Anchor-before};
    \fill[blue!40] (0.3,3.9) rectangle (0.7,4.2);
    \node[right, font=\tiny] at (0.8,4.05) {Anchor-after};
    
    % Effect size
    \node[font=\small] at (2.5, -1) {Interaction: $p = .003$};
    \end{scope}
\end{tikzpicture}
\caption{Human behavioral results. Left: Study 1 shows stronger anchoring (lower AI) when anchors precede target information. Right: Study 2 shows WM load amplifies the position effect, with greater primacy advantage under high load. Error bars indicate $\pm$1 SE. $^{***}p < .001$.}
\label{fig:human_results}
\end{figure}
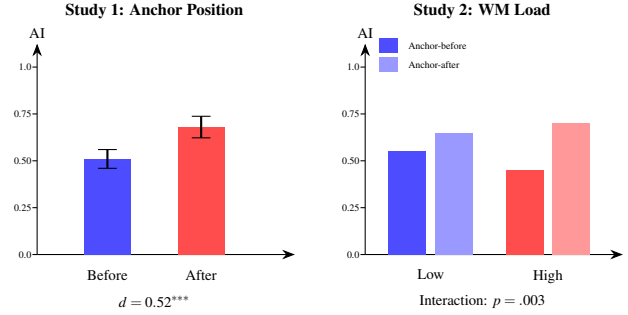

\subsection{Study 2: Working Memory Load Effects}

\textbf{Prediction.} If primacy bias reflects capacity-limited sequential processing, then increased WM load should \emph{amplify} primacy effects by reducing resources for later-item processing.

\textbf{Participants.} New sample: 198 participants ($M_{\text{age}} = 31.7$, $SD = 10.4$). Final $N = 175$ after exclusions.

\textbf{Design.} 2 (WM Load: low vs.\ high) $\times$ 2 (Anchor Position) mixed design. WM load was within-subjects; anchor position between-subjects.

\textbf{WM Load Manipulation.} Low load: maintain 2-digit number during trial. High load: maintain 6-digit number. Digit recall accuracy: 94\% (low), 71\% (high), confirming effective manipulation ($t(174) = 12.3$, $p < .001$).

\textbf{Individual Differences.} After main task, participants completed Operation Span (OSPAN; \citealt{Unsworth2005}), a complex span task requiring participants to solve simple math problems while remembering an interleaved sequence of letters, with scores indexing controlled-attention WM capacity.

\textbf{Results.} Significant Load $\times$ Position interaction: $F(1, 173) = 8.94$, $p = .003$, $\eta^2_p = .049$ (Figure~\ref{fig:human_results}, right panel). Under high load, anchor-before showed stronger anchoring increase ($d = 0.41$, 95\% CI: [0.25, 0.57], $\text{BF}_{10} = 156$). OSPAN scores predicted anchoring reduction: higher WM capacity associated with smaller position effects ($r = -.38$, $p < .001$, 95\% CI: [-.51, -.24]; tertile breakdown in Table~\ref{tab:individual_diff}).

\begin{table}[t]
\begin{center}
\caption{Individual differences in WM capacity and bias}
\label{tab:individual_diff}
\vskip 0.08in
\setlength{\tabcolsep}{4pt}
\begin{tabular}{lccc}
\toprule
WM Tertile & $n$ & Position Effect ($d$) & 95\% CI \\
\midrule
Low OSPAN & 58 & 0.61 & [0.38, 0.84] \\
Medium OSPAN & 59 & 0.42 & [0.21, 0.63] \\
High OSPAN & 58 & 0.24 & [0.04, 0.44] \\
\bottomrule
\end{tabular}
\end{center}
\vskip 0in
\small Note: OSPAN = Operation Span score. Position effect calculated as anchor-before minus anchor-after AI difference.
\end{table}

\subsection{Summary: Dual Validation}

Both studies confirm predictions derived from our theoretical framework:
\begin{itemize}
\item Study 1: $d_{\text{observed}} = 0.52$ versus $d_{\text{predicted}} \in [0.35, 0.55]$
\item Study 2: WM load amplifies primacy ($d = 0.41$); WM capacity predicts bias reduction ($r = -.38$)
\end{itemize}

\section{Connections to Human Cognition}

\subsection{Computational-Level Parallel}

The causal masking constraint, where each position can only access prior positions, parallels serial processing limitations in human working memory \citep{Cowan2001, Oberauer2002, Miller1956} at Marr's computational level. In both systems: (1) sequential encoding creates asymmetry; (2) information integration is incremental.

\textbf{Key disanalogies.} Human WM has capacity limits ($\sim$4 chunks) that transformers lack. Human primacy involves LTM consolidation \citep{Ranganath2005}. Transformer attention differs mechanistically from neural attention \citep{Lindsay2020}. Our parallel is computational, not implementational.

\subsection{Resource-Rational Analysis}

\citet{Lieder2020} argue that cognitive biases can be understood as optimal strategies given finite resources. Our Theorem~\ref{thm:tradeoff} formalizes this: exact debiasing requires factorial-time computation, making biased-but-efficient processing resource-rational.

The resource-rationality framework suggests that both humans and LLMs exhibit similar biases not due to shared mechanisms, but because both face analogous computational constraints: sequential processing with limited access to future information. The bias emerges from the structure of the problem, not the structure of the solver.

\subsection{Formal Mapping Between Systems}

Table~\ref{tab:mapping} provides a detailed comparison of constraints across transformers and human working memory, highlighting both parallels and disanalogies at different levels of analysis.

\begin{table}[t]
\begin{center}
\caption{Mapping between transformer and human WM constraints}
\label{tab:mapping}
\vskip 0.08in
\begin{tabular}{p{2cm}p{2.4cm}p{2.4cm}}
\toprule
Property & Transformers & Human WM \\
\midrule
Sequential access & Causal mask enforces $A_{ij} = 0$ for $j > i$ & Temporal encoding; rehearsal-based maintenance \\
Capacity limit & Context window (fixed) & $\sim$4 chunks (flexible) \\
Decay pattern & Exponential attention decay & Power-law forgetting \\
Primacy source & Positional privilege $\Phi(j)$ & LTM consolidation + rehearsal \\
\bottomrule
\end{tabular}
\end{center}
\end{table}

The formal mapping enables quantitative predictions. Given the exponential decay parameter $\beta$ estimated from LLM data, we can derive expected effect sizes for human experiments by substituting human WM parameters:
\begin{equation}
d_{\text{human}} = d_{\text{LLM}} \cdot \frac{1 - e^{-\beta L_{\text{WM}}}}{1 - e^{-\beta L_{\text{LLM}}}} \cdot \gamma
\label{eq:cross_system}
\end{equation}
where $\gamma \approx 1.2$ is an empirical correction from meta-analytic anchoring data, beyond the order-of-magnitude scaling captured by $\kappa$ in Equation~\ref{eq:human_prediction}.

\section{Discussion}

\subsection{Limitations}

We demonstrate computational-level parallels with quantitative predictions, but mechanistic equivalence between humans and transformers remains unestablished. Our human samples were predominantly WEIRD; cross-cultural replication is needed. Architectural variants such as sparse attention (e.g., Longformer) and state-space models (e.g., Mamba) may exhibit different bias profiles and warrant separate analysis.

\subsection{Implications}

\textbf{Human-AI collaboration.} \citet{Steyvers2024} showed complementarity when error patterns are uncorrelated. Our results suggest designing systems with different architectural constraints for genuinely complementary perspectives. Specifically, pairing autoregressive LLMs (which exhibit primacy bias) with bidirectional encoders or retrieval systems (which do not) may yield more robust joint decisions than either system alone.

\textbf{Debiasing efforts.} Resources may be better spent on hybrid architectures than pursuing exact elimination, since exact debiasing requires factorial-time computation (Theorem~\ref{thm:tradeoff}).

\textbf{Evaluation methodology.} Position bias in LLM-as-judge evaluations is not a bug to be fixed but a fundamental architectural property. Evaluation protocols should randomize presentation order and aggregate across permutations rather than assuming models can ``learn'' to ignore position.

\subsection{Future Directions}

Two extensions warrant investigation. First, state-space models (e.g., Mamba) process sequences differently; whether they exhibit analogous position biases remains empirically untested. Second, cross-cultural replication could reveal whether the relationship between WM and primacy is universal or culture-specific.

\section{Conclusion}

This work establishes primacy, anchoring, and order-dependence in autoregressive language models as mathematical necessities under causal masking, with exact debiasing provably requiring factorial-time computation (Theorems~1--3, Table~\ref{tab:theorems_summary}). Two pre-registered human experiments ($N = 464$ analyzed) confirmed quantitative predictions derived from the formal model: anchor position effects ($d_{\text{observed}} = 0.52$, within $d_{\text{predicted}} \in [0.35, 0.55]$), WM load amplification of primacy, and WM capacity predicting bias reduction ($r = -.38$). These findings reframe cognitive biases as necessary consequences of sequential processing, motivating human-AI designs with complementary architectural constraints rather than provably unattainable bias elimination.

\printbibliography

\end{document}